\documentclass[letterpaper]{article} 
\usepackage{aaai21}  
\usepackage{times}  
\usepackage{helvet} 
\usepackage{courier}  
\usepackage[hyphens]{url}  
\usepackage{graphicx} 
\usepackage{natbib}  
\usepackage{caption} 
\usepackage{algorithm}
\usepackage{algorithmic}
\usepackage{appendix}
\usepackage{bbding}
\usepackage{diagbox}
\usepackage{graphicx}
\usepackage{epstopdf}
\usepackage{float}
\usepackage{subfigure}
\usepackage{multicol}
\usepackage{multirow}
\usepackage{wrapfig}
\usepackage{url}
\usepackage{indentfirst}
\usepackage{amssymb}
\usepackage{amsmath}
\usepackage{delarray}
\usepackage{tabularx}
\usepackage{mathrsfs}
\usepackage{amsmath}
\usepackage{amsthm}
\usepackage{bm}
\usepackage{booktabs}
\usepackage{hyperref}
\usepackage{pifont}
\usepackage{tikz}
\usepackage{textcomp}
\newtheorem{theorem}{Theorem}
\usepackage[T1]{fontenc} 
\usepackage{fancyvrb,cprotect}

\author{
	Sitong Mao \textsuperscript{\rm 1}, 
	Keli Zhang \textsuperscript{\rm 2}, 
	Fu-lai Chung \textsuperscript{\rm 1} \\
}
\title{ Mixed Set Domain Adaptation }
\affiliations{
	\textsuperscript{\rm 1} The Hong Kong Polytechnic University \\
	\textsuperscript{\rm 2} Huawei Noah's Ark Lab \\
	sitong.mao@connect.polyu.hk, zhangkeli1@huawei.com, korris.chung@polyu.edu.hk
}

\begin{document}
	\maketitle
	\begin{abstract}
		In the settings of conventional domain adaptation, categories of the source dataset are from the same domain (or domains for multi-source domain adaptation), which is not always true in reality. In this paper, we propose \textbf{\textit{Mixed Set Domain Adaptation} (MSDA)}. Under the settings of MSDA, different categories of the source dataset are not all collected from the same domain(s). For instance, category $1\sim k$ are collected from domain $\alpha$ while category $k+1\sim c$ are collected from domain $\beta$. Under such situation, domain adaptation performance will be further influenced because of the distribution discrepancy inside the source data. A feature element-wise weighting (FEW) method that can reduce distribution discrepancy between different categories is also proposed for MSDA. Experimental results and quality analysis show the significance of solving MSDA problem and the effectiveness of the proposed method.
	\end{abstract}
	
	\section{Introduction}
	Domain adaptation focuses on adapting a model to the target data with the help of source data which are from different but related domains \cite{farseev2017cross} \cite{liu2008evigan} \cite{mcclosky2006reranking} \cite{daume2009frustratingly} \cite{saenko2010adapting} \cite{long2015learning} \cite{long2017deep} \cite{gong2012geodesic} \cite{weston2012deep} \cite{pan2011domain} \cite{tzeng2014deep} \cite{yu2018multi} \cite{gholami2020unsupervised}. The distribution discrepancy between source data and target data, which is referred to as domain shift, will influence the adaptation performance. In typical settings of domain adaptation, categories of the source dataset are collected from the same domain (or the same set of domains for multi-source domain adaptation \cite{sun2011two} \cite{sun2015survey} \cite{peng2019moment}). However, such settings are limited in the reality where not all categories are always collected from the same domain. For example, pictures of class $1$ of the source dataset are taken on sunny days while class $2$ are taken on cloudy days. Then when a model is trained to classify these images, it may focus on the weather (i.e., domain) information besides their content information, which creates obstacle of extracting effective features during domain adaptation process. To address this problem, we propose Mixed Set Domain Adaptation (MSDA), under the settings of which categories of source dataset are from several (e.g., two in this paper) different domains as illustrated in Figure \ref{fig:msda}.
	
	\begin{figure}
		\includegraphics[width=\linewidth]{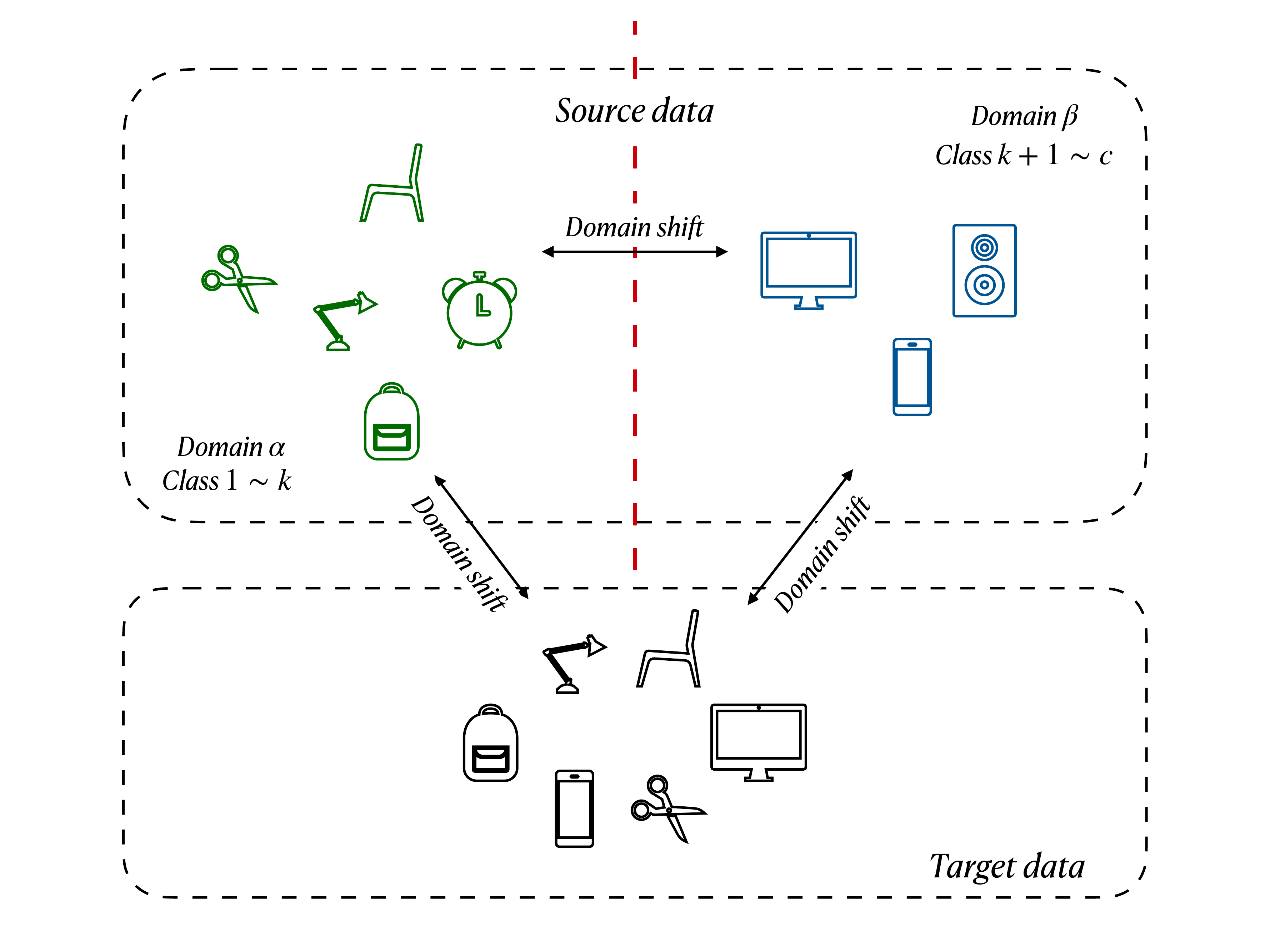}
		\caption{Illustration of Mixed Set Domain Adaptation.}
		\label{fig:msda}
	\end{figure}
	
	The objective of previous domain adaptation methods is to reduce the distribution discrepancy between the source data and the target data. Combining with deep networks \cite{yosinski2014transferable}, some methods utilize statistical metrics to bound the distribution discrepancy between source features and target features, such as ``MMD'' \cite{long2015learning} and ``JMMD'' \cite{long2017deep}. Recently, adversarial learning has achieved remarkable performance at mapping data to similar distributions \cite{goodfellow2014generative} \cite{tzeng2017adversarial}. Deep networks with adversarial learning architecture embedded generally has two components: feature extractor and discriminator. The discriminator aims to distinguish which domain a deep feature is from while the feature extractor tries to confuse the discriminator \cite{ganin2016domain} \cite{long2018conditional}. Then an optimum will be reached when the source features and the target features follow the same distribution.
	
	Different from the settings of conventional domain adaptation, not all categories of source data are from the same doamin(s) in mixed set domain adaptation (MSDA) problem. Hence, we need to address the \textbf{\textit{domain shift between different categories}} in source dataset besides bridging target data and source data. The influence of the domain shift between different categories of source dataset will be further discussed in ``Problem Setting'' section and ``Experiments'' section. The main challenge of MSDA is how to reduce the distribution discrepancy caused by domain shift instead of that caused by the difference of contents between different categories. Methods that reduce conditional distribution discrepancy tend to map data which belong to the same category close. Therefore, marginal distribution discrepancy is considered in this paper when reducing domain shift between different categories in MSDA problems. However, if marginal distribution discrepancy is minimized directly, we can not make sure which parts of the features are mapped close while we want to focus on those that can represent domain instead of content. In this paper, we propose a ``feature element-wise weighting'' (FEW) method to assign different weights to different elements of the feature vector during the adversarial training process according to how important they are for predicting labels, which can then address the domain shift between different categories. The main contributions of this paper are: 
	\begin{itemize}
		\item[i)] A new setting of domain adaptation is proposed: Mixed Set Domain Adaptation (MSDA). MSDA does not require different categories of the source dataset to be from the same domain(s), which is more common in reality than typical domain adaptation.
		\item[ii)] A feature element-wise weighting method is proposed and devised to address domain shift between different categories, which outperforms traditional domain adaptation methods when dealing with MSDA problems.
		\item[iii)] Various experiments were conducted to evaluate the significance of MSDA and the effectiveness of the proposed method is demonstrated.
	\end{itemize}
	
	
	\section{Related Work}
	Previous methods for typical deep domain adaptation have achieved remarkable performance. Some of them utilize statistical metrics to measure the distribution discrepancy between source domain and target domain, such as ``MMD'' \cite{long2015learning} which limits the marginal distribution discrepancy, ``JMMD'' \cite{long2017deep} which bounds the conditional distribution discrepancy, and ``MJKD''  \cite{mao2018deep} which processes data in different categories separately. Recently, by using adversarial learning architecture, the distance between source distribution and target distribution can be further reduced and significant improvements have been achieved. In \cite{ganin2016domain}, deep features are sent to the discriminator to minimize the marginal distribution discrepancy without additional label information. To further improve the adaptation performance, the tensor products of deep features and the $softmax$ probabilities distributed over each class predicted by the deep neural networks are sent to the discriminator so that the conditional distribution discrepancy can be reduced \cite{long2018conditional} \cite{zhang2019domain} \cite{wang2019transferable} \cite{zhang2019bridging}.
	
	\begin{figure*}[t]
		\centering
		\includegraphics[width = 0.8\linewidth]{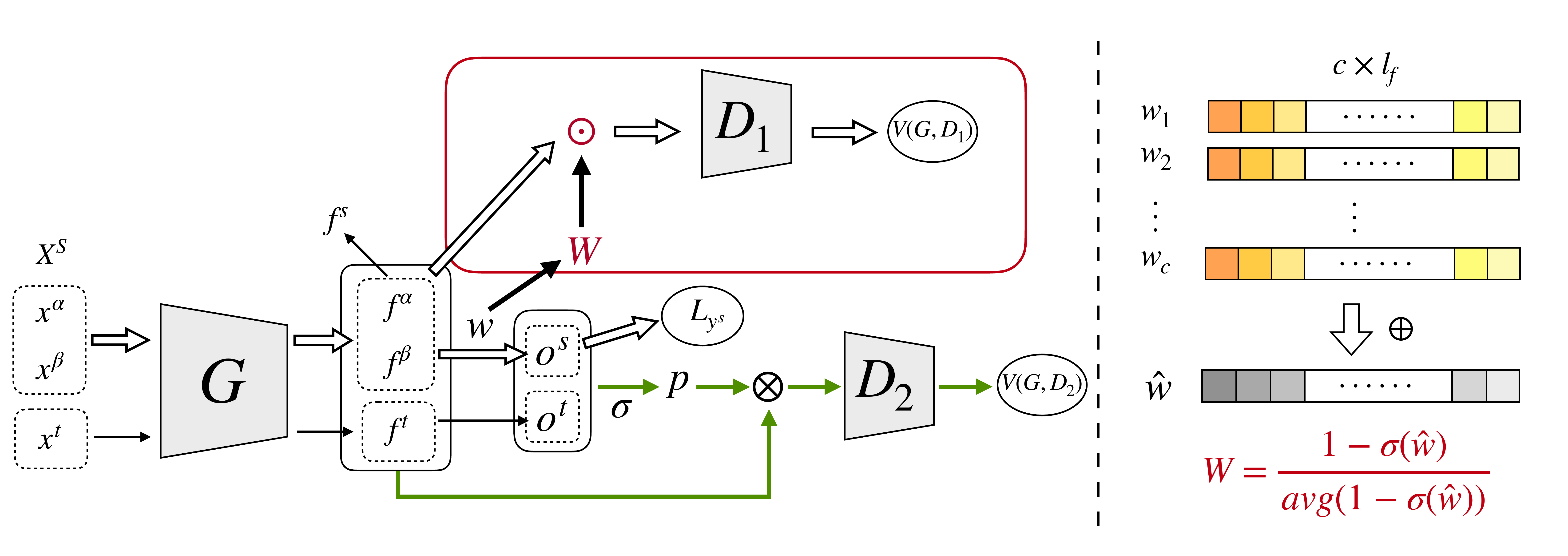}
		\caption{Illustration of the proposed method. The left panel is the proposed architecture. The right panel is the way of computing $W$ by $w$.}\label{fig:msda_method}
	\end{figure*}
	
	Different from the settings of conventional domain adaptation, the mixed set domain adaptation to be proposed focuses on domain adaptation problem where different categories in source dataset are from several (e.g., two in this paper) domain(s). For instance, category $1\sim k$ are from domain $\alpha$ while category $k+1\sim c$ are collected from domain $\beta$. The performance of previous domain adaptation methods may decline under the influence of the domain shift between different categories of the source data. In this paper, we propose a feature element-wise weighting method to map different domains of different categories close. In our proposed method, different elements of the feature vector are adjusted in different strength to achieve the goal that we want to bridge the domain shift instead of confusing the categories. 
	
	\section{Mixed Set Domain Adaptation}
	In this section, we first introduce the problem settings of Mixed Set Domain Adaptation. Then, the feature element-wise weighting (FEW) method is presented in detail.
	
	\subsection{Problem Setting}
	In unsupervised mixed set domain adaptation, labeled source dataset $S = \{X^S, Y^S\}$ and unlabeled target dataset $T = \{X^T\}$ are given. In the context of this paper, $S$ contains two domains, i.e., category $1\sim k$ are from domain $\alpha$: $S^\alpha = \{X^{\alpha}, Y^{\alpha}\}$ while category $k+1\sim c$ are collected from domain $\beta$: $S^\beta = \{X^{\beta}, Y^{\beta}\}$. Here $X^\alpha = \{x^{\alpha}_i|i=1,2,\cdots, n^{\alpha}\}$ and $X^\beta = \{x^{\beta}_i|i=1,2,\cdots, n^{\beta}\}$, $n^\alpha$ and  $n^\beta$ are the number of the data of domain $\alpha$ and domain $\beta$ respectively. $Y^\alpha = \{y^\alpha_i|y^\alpha_i \in \{1,2,\dots,k\}\}$ and $Y^\beta = \{y^\beta_i|y^\beta_i \in \{k+1,k+2,\dots,c\}\}$. Then, we have $S = S^\alpha \cup S^\beta = \{X^\alpha\cup X^\beta, Y^\alpha\cup Y^\beta\}$. 
	
	$P(X^{\alpha}_y)$ (resp. $P(X^{\beta}_y)$) is used to denote the distribution of the content information contained in $X^\alpha$ (resp. $X^\beta$) and $P(X^{\alpha}_d)$ (resp. $P(X^{\beta}_d)$) denotes the distribution of the domain information contained in $X^\alpha$ (resp. $X^\beta$). Under the traditional settings of domain adaptation, domain $\alpha$ and domain $\beta$ are the same, thus we have $P(X^{\alpha}_y) \neq P(X^{\beta}_y)$ but $P(X^{\alpha}_d) = P(X^{\beta}_d)$. While in the settings of MSDA, we have $P(X^{\alpha}_y) \neq P(X^{\beta}_y)$ and $P(X^{\alpha}_d) \neq P(X^{\beta}_d)$. Thus not only label information $X_y$ will be considered but also domain information $X_d$ would be taken into account when classifying $X^\alpha$ and $X^\beta$ in MSDA problems. As a result, some content-irrelevant part of an image will be focused on and consequently influences the performance of adaptation. Heat maps can reflect which area the model focuses on to some degree, thus we show an example of this problem in Figure \ref{fig:misclassified_example} utilizing ``Grad-Cam'' \cite{selvaraju2017grad}. Figure \ref{fig:misclassified_example} (a) is output by the model trained on typical task W$\rightarrow$D using CDAN \cite{long2018conditional} which classifies the target image correctly (i.e., class $26$, scissors). Then as shown in Figure \ref{fig:misclassified_example} (b), using CDAN on MSDA transfer task \{A,W\}$\rightarrow$D, the same target image is misclassified to class $17$ (i.e., mug) because of focusing on too much content-irrelevant area (the red part) compared with Figure \ref{fig:misclassified_example} (a). And after applying our proposed method ``FEW'', the network focuses more on the object and classifies it correctly as shown in Figure \ref{fig:misclassified_example} (c). 
	
	\begin{figure}[h]
		\centering
		\subfigure[CDAN(Tra.)]{
			\includegraphics[width = 0.26\linewidth]{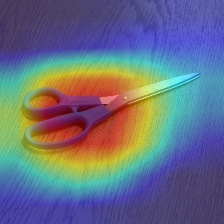}
		}\hspace{0.1in}
		\subfigure[CDAN]{
			\includegraphics[width = 0.26\linewidth]{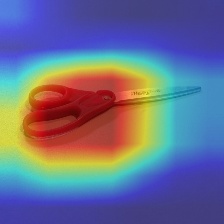}
		}\hspace{0.1in}
		\subfigure[FEW]{
			\includegraphics[width = 0.26\linewidth]{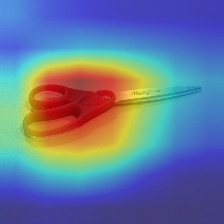}
		}
		\caption{A case where domain shift between different categories influences the classification result.}
		\label{fig:misclassified_example}
	\end{figure}
	Figure \ref{fig:misclassified_example} (a), (b), and (c) illustrates the significance of solving MSDA problems.
	
	\subsection{Feature Element-Wise Weighting}
	Since domain shift between different categories of source data influences the domain adaptation performance, it should be addressed together with the distribution discrepancy between the source data and the target data. Here we want to address domain shift between different categories, thus instead of minimizing their conditional distribution discrepancy which focuses on domain shift between the same categories, we choose to reduce their marginal distribution discrepancy as shown in eq. \ref{eq:mar_adv}. 
	\begin{equation}
	\begin{aligned}
	\min & \max V(G,D_1) = \mathbb{E}_{x^{\alpha}\sim p_{\alpha}(x)} [log D_1 (G^{f}(x^{\alpha})] \\
	& + \mathbb{E}_{x^{\beta}\sim p_{\beta}(x)}[log(1 - D_1 (G^{f}(x^{\beta})))] \\
	\end{aligned}
	\label{eq:mar_adv}
	\end{equation}
	Here, source data of domain $\alpha$ follows distribution $p_\alpha(x)$: $x^{\alpha}\sim p_{\alpha}(x)$ and source data of domain $\beta$ follows distribution $p_\beta(x)$: $x^{\beta}\sim p_{\beta}(x)$. $G$ is the feature extractor and $D_1$ is the discriminator. $G^f (x^{\alpha})$ and $G^f (x^{\beta})$ denote the deep features of source domain $\alpha$ and source domain $\beta$ extracted by $G$ respectively. $D_1(\cdot)$ denotes the probability of a source data belonging to domain $\alpha$ predicted by $D_1$.
	
	However, minimizing marginal distribution discrepancy may cause category confusion by mapping features from different categories close. Since we want to bridge the domain shift instead of confusing categories, we propose to give larger weights to feature elements that contain more domain information than elements which contain more content information. Then, the feature element-wise weighted adversarial learning process becomes 
	\begin{equation}
	\begin{aligned}
	\min & \max V(G,D_1)  \\ 
	& = \mathbb{E}_{x^{\alpha}\sim p_{\alpha}(x)} [log D_1 (W \odot G^{f}(x^{\alpha})] + \\
	& \qquad \mathbb{E}_{x^{\beta}\sim p_{\beta}(x)}[log(1 - D_1 (W \odot G^{f}(x^{\beta})))] \\
	& = \mathbb{E}_{f ^\alpha\sim G_{\alpha}(f)}[log D_1 (W \odot f^\alpha)] + \\
	& \qquad \mathbb{E}_{f^\beta\sim G_{\beta}(f)}[log(1 - D_1 (W \odot f^\beta))]
	\end{aligned}
	\label{eq:mar_adv_weight}
	\end{equation}
	Here $G_\alpha(f)$ and $G_\beta(f)$ denote the distributions of deep features $f^{\alpha}$ and $f^{\beta}$ respectively. Operator ``$\odot$'' denotes the element-wise multiplication. The dimension of deep feature vector $f$ is $l_f$, i.e., $f = [f_1, f_2,\dots, f_{l_f}]$. Then the weight vector ``$W$'' should have the same dimension as $f$. In our proposed method, ``$W$'' is computed from the weight matrix ``$w$'' of the last fully connected layer as shown in the left panel of Figure \ref{fig:msda_method}. Use $c$ to denote the number of classes. Then, ``$w$'' is a matrix of size $c \times l_f$. In neural networks, the inner product of a feature vector and the $i$th row of $w$ (i.e., $w_i$) outputs the probability of this feature belonging to class $i$, i.e., $p_i = w_{i1}f_1+w_{i2}f_2+\dots +w_{il_f}f_{l_f}$. Thus, elements of $w_i$ can be used to measure how much a feature element weights for being classified to class $i$. For example, $w_{ik}$ is the $k$th element of $w_i$, then it can reflect how much the $k$th element $f_k$ of a feature vector weights for being classified to the $i$th category. Then, we assume that those feature elements which have small weights for all categories are more related to domain information instead of category information. Following this assumption, $W$ is defined as
	\begin{equation}
	\begin{aligned}
	W & = \frac{1 - \sigma(\hat{w})}{avg(1 - \sigma(\hat{w}))}, \\ 
	& \hat{w}_k = \sum_{i}^{c} w_{ik}
	\end{aligned}
	\label{eq:ele_weight}
	\end{equation}	
	Here $\hat{w} = w_1 \oplus w_2 \oplus \dots \oplus w_c, \hat{w} \in \mathbb{R}^{1\times l_f}$ is the element-wise sum by column of $w$ as shown in the right panel of Figure \ref{fig:msda_method} and $\sigma$ denotes the softmax function which is used to convert elements of $\hat{w} = [\hat{w}_1, \hat{w}_2, \cdots, \hat{w}_{l_f}]$ to probabilities. ``$avg(1 - \sigma(\hat{w}))$'' is the average of each element in $1 - \sigma(\hat{w})$, i.e., $W$ is normalized. As shown in eq. \ref{eq:ele_weight}, when the sum of the $k$th column of $w$ (i.e., $\hat{w}_k$) is small, then the $k$th element of each feature vector will be considered to be less important for predicting labels, i.e., it contains more domain information $X_d$. Thus, the $k$th feature element will be given larger weight when mapping domain $\alpha$ and domain $\beta$ close. Since $w$ is not stable and accurate enough at the beginning of the training process, we use $\tilde{W}$ instead of $W$ in our experiments as shown in eq. \ref{eq:tilde_w}.
	\begin{equation}
	\begin{aligned}
	\tilde{W} = (W- 1)\times \delta + 1
	\end{aligned}
	\label{eq:tilde_w}
	\end{equation}
	Here $\delta = \frac{2.0}{1.0+\exp(-\eta\times\frac{iter}{max\_iter})}$ \cite{long2018conditional}, $\eta$ is set to $10$, $iter$ is the current number of iteration, and $max\_iter$ is set to $10000$ in our experiments. Then, the gradients back propagated to $f_k$ (i.e., the $k$th element of each feature vector) from eq. \ref{eq:mar_adv_weight} is
	\begin{equation}
	\begin{aligned}
	\frac{\partial V(G,D_1)}{\partial f_k} = \tilde{W_k} \cdot \frac{\partial V(G,D_1)}{\partial (\tilde{W_k} f_k)}
	\end{aligned}
	\label{eq:grad_tilde_w}
	\end{equation}
	where $\tilde{W_k}$ is the $k$th element of $\tilde{W}$.
	
	\begin{table*}[t]
		\centering
		\caption[LoF entry]{Accuracy on Office-31 under standard settings of Mixed Set domain adaptation. }
		\renewcommand\arraystretch{1.0}
		\small
		\begin{tabular}{p{2.3cm}<{\centering}|p{0.8cm}<{\centering}|p{1.5cm}<{\centering}p{1.5cm}<{\centering}p{1.5cm}<{\centering}p{1.5cm}<{\centering}p{1.5cm}<{\centering}p{1.5cm}<{\centering}|p{0.8cm}<{\centering}}
			\hline
			Method & MSDA & A$\rightarrow$W & A$\rightarrow$D & D$\rightarrow$W & D$\rightarrow$A & W$\rightarrow$D & W$\rightarrow$A & Avg. \\
			\hline
			ResNet-50 & \ding{55}  & $68.4\pm0.2$ & $68.9\pm0.2$ & $96.7\pm0.1$ & $62.5\pm0.3$ & $99.3\pm0.1$ & $60.7\pm0.3$ & $76.1$ \\
			DANN & \ding{55} & $82.0\pm0.4$ & $79.7\pm0.4$ & $96.9\pm0.2$ & $68.2\pm0.4$ & $99.1\pm0.1$ & $67.4\pm0.5$ & $82.2$ \\
			CDAN & \ding{55} & $93.1\pm0.2$ & $89.8\pm0.3$ & $98.2\pm0.2$ & $70.1\pm0.4$ & $100\pm0.0$ & $68.0\pm0.4$ & $86.6$ \\
			FEW & \ding{55} & $91.5\pm1.9$ & $85.2\pm1.5$ & $98.9\pm0.0$ & $68.7\pm1.6$ & $100\pm0.0$ & $66.5\pm0.9$ & $85.1$ \\
			\hline
			Method & MSDA & \{A,D\}$\rightarrow$W & \{A,W\}$\rightarrow$D & \{D,A\}$\rightarrow$W & \{D,W\}$\rightarrow$A & \{W,A\}$\rightarrow$D & \{W,D\}$\rightarrow$A & Avg. \\
			\hline
			ResNet-50 & \ding{51} & $86.0\pm0.6$ & $90.7\pm0.5$ & $92.7\pm0.1$ & $62.5\pm0.6$ & $92.5\pm0.3$ & $59.8\pm0.3$ & $80.7$ \\
			DANN & \ding{51} & $92.9\pm1.2$ & $91.1\pm0.5$ & $93.4\pm0.7$ & $66.7\pm1.1$ & $94.0\pm0.5$ & $65.0\pm1.8$ & $83.9$ \\
			CDAN & \ding{51} & $95.5\pm0.4$ & $92.6\pm0.3$ & $97.1\pm0.4$ & $68.5\pm0.3$ & $97.0\pm0.5$ & $69.0\pm0.5$ & $86.6$ \\
			\hline
			CDAN+DANN & \ding{51} & $95.4\pm 0.3$ & $94.8\pm0.6$ & $\bm{97.8\pm 0.2}$ & $68.3\pm0.7$ & $97.1\pm0.1$ & $69.5\pm0.9$ & $87.1$ \\
			FEW & \ding{51} & $\bm{95.9\pm0.2}$ & $\bm{95.0\pm0.5}$ & $97.5\pm0.1$ & $\bm{68.7\pm0.5}$ & $\bm{97.3\pm0.3}$ & $\bm{70.4\pm0.4}$ & $\bm{87.5}$ \\
			\hline
		\end{tabular}
		\label{tab:office_resnet50}
	\end{table*}
	\begin{table*}[t]
		\setlength{\belowcaptionskip}{0.5cm}
		\centering
		\caption[LoF entry]{Accuracy on Office-Home under standard settings of Mixed set domain adaptation.}
		\renewcommand\arraystretch{1.0}
		\small
		\begin{tabular}{p{2.3cm}<{\centering}|p{0.8cm}<{\centering}|p{1.5cm}<{\centering}p{1.5cm}<{\centering}p{1.5cm}<{\centering}p{1.5cm}<{\centering}p{1.5cm}<{\centering}p{1.6cm}<{\centering}|p{0.8cm}<{\centering}}
			\hline
			Method & MSDA & Ar$\rightarrow$Cl & Pr$\rightarrow$Cl & Ar$\rightarrow$Pr & Cl$\rightarrow$Ar & Cl$\rightarrow$Pr &  Cl$\rightarrow$Rw & Avg. \\
			\hline
			ResNet50 & \ding{55} & $45.1$ & $41.4$ & $64.9$ & $54.3$ & $60.4$ & $62.8$ & $54.8$ \\
			DANN & \ding{55} & $47.3$ & $45.8$ & $64.6$ & $53.7$ & $62.1$ & $64.1$ & $56.3$  \\			
			CDAN & \ding{55} & $49.0$ & $48.3$ & $69.3$ & $54.4$ & $66.0$ & $68.4$ & $59.2$ \\
			FEW & \ding{55} & $49.1$ & $44.0$ & $64.6$ & $51.2$ & $61.8$ & $60.4$ & $55.2$ \\
			\hline		
			Method & MSDA & \{Ar,Cl\}$\rightarrow$Pr & \{Ar,Pr\}$\rightarrow$Cl & \{Cl,Pr\}$\rightarrow$Ar & \{Cl,Rw\}$\rightarrow$Ar & \{Rw,Cl\}$\rightarrow$Ar & \{Pr,Cl\}$\rightarrow$Rw & Avg. \\
			\hline
			ResNet50 & \ding{51} & $61.0\pm0.7$ & $42.5\pm1.0$ & $51.1\pm0.6$ & $54.6\pm0.2$ & $59.2\pm0.4$ & $63.9\pm0.3$ & $51.7$ \\
			DANN & \ding{51} & $60.3\pm0.4$ & $45.6\pm0.2$ & $47.3\pm0.4$ & $55.8\pm0.7$ & $57.2\pm0.5$ & $65.8\pm0.5$ & $55.3$  \\			
			CDAN & \ding{51} & $59.7\pm0.9$ & $45.9\pm0.5$ & $52.0\pm0.5$ & $56.4\pm0.2$ & $\bm{63.9\pm0.3}$ & $67.8\pm0.6$ & $57.6$ \\
			\hline
			CDAN+DANN & \ding{51} & $62.4\pm0.7$ & $\bm{47.1\pm0.2}$ & $\bm{52.4\pm0.4}$ & $56.3\pm0.6$ & $62.2\pm0.3$ & $67.7\pm0.3$ & $58.0$ \\
			FEW & \ding{51} & $\bm{62.7\pm0.7}$ & $\bm{47.1\pm0.8}$ & $\bm{52.4\pm0.6}$ & $\bm{56.8\pm0.1}$ & $62.8\pm0.3$ & $\bm{68.3\pm0.2}$ & $\bm{58.4}$ \\
			\hline
			\multicolumn{9}{c}{} \\
			\multicolumn{9}{c}{} \\
			\hline
			Method & MSDA & \{Pr,Ar\}$\rightarrow$Cl & \{Rw,Pr\}$\rightarrow$Ar & \{Ar,Rw\}$\rightarrow$Cl & \{Cl,Rw\}$\rightarrow$Pr & \{Pr,Rw\}$\rightarrow$Ar & \{Rw,Ar\}$\rightarrow$Cl & Avg. \\
			\hline			
			CDAN & \ding{51} & $46.0\pm0.7$ & $\bm{62.3\pm0.3}$ & $47.6\pm0.5$ & $66.4\pm1.3$ & $58.2\pm0.7$ & $51.4\pm0.5$ & $55.3$ \\
			\hline
			CDAN+DANN & \ding{51} & $\bm{47.1\pm0.8}$ & $61.8\pm0.6$ & $\bm{48.2\pm0.3}$ & $65.8\pm0.3$ & $57.5\pm0.6$ & $48.7\pm0.8$ & $54.9$ \\
			FEW & \ding{51} & $\bm{47.1\pm0.6}$ & $62.0\pm0.3$ & $\bm{48.2\pm0.3}$ & $\bm{67.1\pm0.6}$ & $\bm{58.5\pm0.2}$ & $\bm{51.5\pm0.3}$ & $\bm{55.7}$ \\
			\hline
		\end{tabular}
		\label{tab:officehome_resnet50}
	\end{table*}
	
	\subsection{Overall Training Process}
	The end-to-end architecture of the propose method is illustrated in Figure \ref{fig:msda_method}. $f^s$ (resp. $f^t$) is the source (resp. target) features extracted by $G$. $o^s = f^s w$ and $o^t = f^t w$ denote the source and target output of the last fully connected layer respectively. In Figure \ref{fig:msda_method}, discriminator $D_1$ is used to distinguish which source domain an element-wise weighted source feature comes from and discriminator $D_2$ is used to tell whether a feature comes from the source dataset or the target dataset. The feature generator $G$ aims to extract deep features that can confuse $D_1$ and $D_2$. The loss function of $D_1$ is given in eq. \ref{eq:mar_adv_weight} and $D_2$ is trained in the same way as proposed in \cite{long2018conditional}:
	\begin{equation}
	\begin{aligned}
	\min\limits_G & \max\limits_{D} V(G,D_2) = \mathbb{E}_{(f^s,p^s)\sim G_{s}}[logD_2(f^s \otimes p^s)] \\
	& + \mathbb{E}_{(f^t,p^t)\sim G_{t}}[log(1-D_2(f^t \otimes p^t))].
	\end{aligned}
	\label{eq:cdan}
	\end{equation}
	Here $p^s = \sigma(o^s) = \sigma(f^s w)$ and $p^t = \sigma(o^t) = \sigma(f^t w)$ are the $softmax$ probabilities distributed over each class, and ``$\otimes$'' denotes the tensor product operation. $(f^s,p^s)$ and $(f^t,p^t)$ follow distribution $G_s$ and $G_t$ respectively. The classification loss of the source data is
	\begin{equation}
	-\mathcal{L}_{y^s} = \frac{1}{n^s}\sum\limits^{n^s}_{i}log(p^s_{i,y_i})
	\label{eq:source_loss}
	\end{equation}
	where $p^s_{i,y_i}$ is the probability of the $i$th source data belonging to its ground truth category $y^s_i$. Then, the integrated loss function is
	\begin{equation}
	\begin{aligned}
	\min \limits_G V(G,D_1) + V(G,D_2) + \mathcal{L}_{y^s} \\
	\max \limits_{D_1, D_2} V(G,D_1) + V(G,D_2)
	\end{aligned}
	\label{eq:total_loss}
	\end{equation}
	
	\begin{table*}[t]
		\centering
		\caption[LoF entry]{Accuracy under different category splitting strategies of MSDA.}
		\small
		\begin{tabular}{p{2.0cm}<{\centering}p{1.5cm}<{\centering}p{1.5cm}<{\centering}p{1.0cm}<{\centering}|p{1.6cm}<{\centering}p{1.6cm}<{\centering}p{1.0cm}<{\centering}}
			\hline
			\multirow{2}{*}{Method} &
			\multicolumn{3}{c|}{ Office-31} & \multicolumn{3}{c}{ Office-Home }  \cr 
			& \{A,D\}$\rightarrow$W & \{W,D\}$\rightarrow$A & Avg. & \{Rw,Cl\}$\rightarrow$Ar & \{Pr,Rw\}$\rightarrow$Ar & Avg. \\
			\hline
			CDAN & $98.2\pm0.2$ & $66.8\pm0.3$ & $82.5$ & $49.5\pm0.6$ & $63.9\pm0.9$ & $56.7$ \\
			CDAN+DANN & $98.2\pm0.2$ & $67.0\pm0.8$ & $82.6$ & $50.0\pm0.0$ & $63.8\pm0.6$ & $56.9$ \\
			FEW & $\bm{98.4\pm0.3}$ & $\bm{67.9\pm0.3}$ & $\bm{83.2}$ & $\bm{50.9\pm0.2}$ & $\bm{64.5\pm0.6}$ & $\bm{57.7}$ \\
			\hline
		\end{tabular}
		\label{tab:office_diff}
	\end{table*}
	
	The overall training process is given in Algorithm~\ref{algo_1}, where ``$n$'' denotes the total number of iterations. 
	
	{\renewcommand\baselinestretch{0.8}\selectfont
		\renewcommand{\algorithmicrequire}{\textbf{Input:}}
		\renewcommand{\algorithmicensure}{\textbf{Procedure:}}
		\begin{algorithm}
			\caption{Feature Element-Wise Weighting (FEW)}\label{algo_1}
			\begin{algorithmic}[1]		
				\REQUIRE
				\STATE Source dataset $\mathcal{S} = \mathcal{S}^\alpha \cup\mathcal{S}^\beta$; Target data $\mathcal{T}$;
				\ENSURE ~~\\
				\FOR {$i=1:n$}
				\STATE	Extract deep features $f^\alpha$, $f^\beta$, $f^s$, $f^t$, $p^s$ and $p^t$; \\
				\STATE	Extract $w$, then: \\
				\quad $W = \frac{1 - \sigma(\hat{w})}{avg(1 - \sigma(\hat{w}))}$, where $\hat{w}_k = \sum_{i}^{c} w_{ki}$.
				\STATE Following eq. \ref{eq:mar_adv_weight} to bound distribution discrepancy between $\bm{Wf^\alpha}$ and $\bm{Wf^\beta}$; \\
				\STATE Following eq. \ref{eq:cdan} to bound distribution discrepancy between $\bm{f^s \otimes p^s}$ and $\bm{f^t \otimes p^t}$; \\
				\STATE Update $G$, $D_1$, and $D_2$ by eq. \ref{eq:total_loss}; \\
				\ENDFOR
				\STATE Test on the target data; \\
				\renewcommand{\algorithmicensure}{\textbf{Output:}}
				\ENSURE ~~\\
				\STATE Classification accuracy of the target data.
			\end{algorithmic}
		\end{algorithm}
		\par}
	
	\section{Experiments}
	In this section, we first describe the setup of our experiments. Then we evaluate the proposed method and several previous state-of-the-art methods. Qualitative analysis such as convergence and heat map are also given to further illustrate the significance of MSDA and the effectiveness of the proposed FEW method.
	
	\subsection{Experimental Setup}
	\textbf{Datasets.}
	The \textbf{Office-31} dataset\cprotect\footnote{https://people.eecs.berkeley.edu/\Verb!~!jhoffman/domainadapt/\\ \Verb!#!datasets$\_$code} \cite{saenko2010adapting} contains images originated from three domains: Amazon (A), Webcam (W), and DSLR (D). These three domains consist of the same 31 categories. By using Office-31, we can evaluate the proposed method on 6 MSDA transfer tasks: \{A,D\}$\rightarrow$W, \{A,W\}$\rightarrow$D, \{D,A\}$\rightarrow$W, \{D,W\}$\rightarrow$A, \{W,A\}$\rightarrow$D, and \{W,D\}$\rightarrow$A.
	
	The \textbf{Office-Home}\footnote{http://hemanthdv.org/OfficeHome-Dataset/}~\cite{venkateswara2017deep} dataset consists of 4 domains of everyday objects in office and home settings: Artistic (Ar), Clip Art (Cl), Product (Pr) and Real-World (Rw). There are 65 categories in each domain and more than 15,000 images in total. Compared with Office-31, this is a more challenging dataset for domain adaptation evaluation because each domain contains more categories and different domains have significant domain shifts. For this dataset, we selected 12 MSDA transfer tasks for evaluation: \{Ar,Cl\}$\rightarrow$Pr, \{Ar,Pr\}$\rightarrow$Cl, \{Cl,Rw\}$\rightarrow$Ar, \{Rw,Cl\}$\rightarrow$Ar, \{Cl,Pr\}$\rightarrow$Ar, \{Pr,Ar\}$\rightarrow$Cl, \{Pr,Cl\}$\rightarrow$Rw, \{Rw,Pr\}$\rightarrow$Ar, \{Ar,Rw\}$\rightarrow$Cl, \{Cl,Rw\}$\rightarrow$Pr, \{Pr,Rw\}$\rightarrow$Ar, and \{Rw,Ar\}$\rightarrow$Cl.
	
	In this paper, we specify a standard setting of MSDA as follows. For dataset Office-31, category $1\sim 20$ of the source data are from domain $\alpha$ while category $21\sim 31$ are from domain $\beta$. For Office-Home, category $1\sim 40$ of the source data are from domain $\alpha$ and category $41\sim 65$ come from domain $\beta$. For example, in task \{Ar,Cl\}$\rightarrow$Pr, category $1\sim40$ are from Ar and category $41\sim65$ are from Cl. In addition, other category splitting strategies are also evaluated in our experiments. 
	
	\textbf{Network Architecture.}
	The feature extractor $G$ in our experiments was built based on the architecture of ResNet50~\cite{he2016deep} pre-trained on ImageNet~\cite{deng2009imagenet}. To fairly compare with other adversarial methods \cite{long2018conditional}, a $bottleneck$ layer with size 256 is added before the last full-connected layer. The two discriminators used in our experiments both consist of three fully connected layers. The size of the first two layers are 1024 followed by ReLU activation layer and dropout layer while the dimension of final output is 1.
	
	\textbf{Training Process.}
	Models of our experiments were trained based on the framework \textbf{Pytorch} using GPU \textbf{Tesla V100 32G} on Linux system. Following the standard fine-tuning procedure, the learning rates of the first few layers were set to a small number to slightly tune the parameters initialized from the pre-trained model. The learning rate for the other layers like $bottleneck$ and the last fully connected layer can be set larger, typically 10 times that of the lower layers. We used the stochastic gradient descent (SGD) update strategy with a momentum of 0.9. The base learning rates $base\_lr$ for all tasks of Office-31 dataset and Office-Home dataset are set to 0.001. The learning rate was changed by the following strategy during the training process: $base\_lr\times (1+\gamma \times iter)^{-power}$, where $power$ was set to $0.75$ and $\gamma$ was set to $0.001$ throughout all experiments and $iter$ is the current number of iterations. The total number of iterations were set to $15000$ for all tasks.
	
	\begin{figure*}[ht]
		\centering
		\subfigure[\{A,W\}$\rightarrow$D]{
			\includegraphics[width=0.3\linewidth]{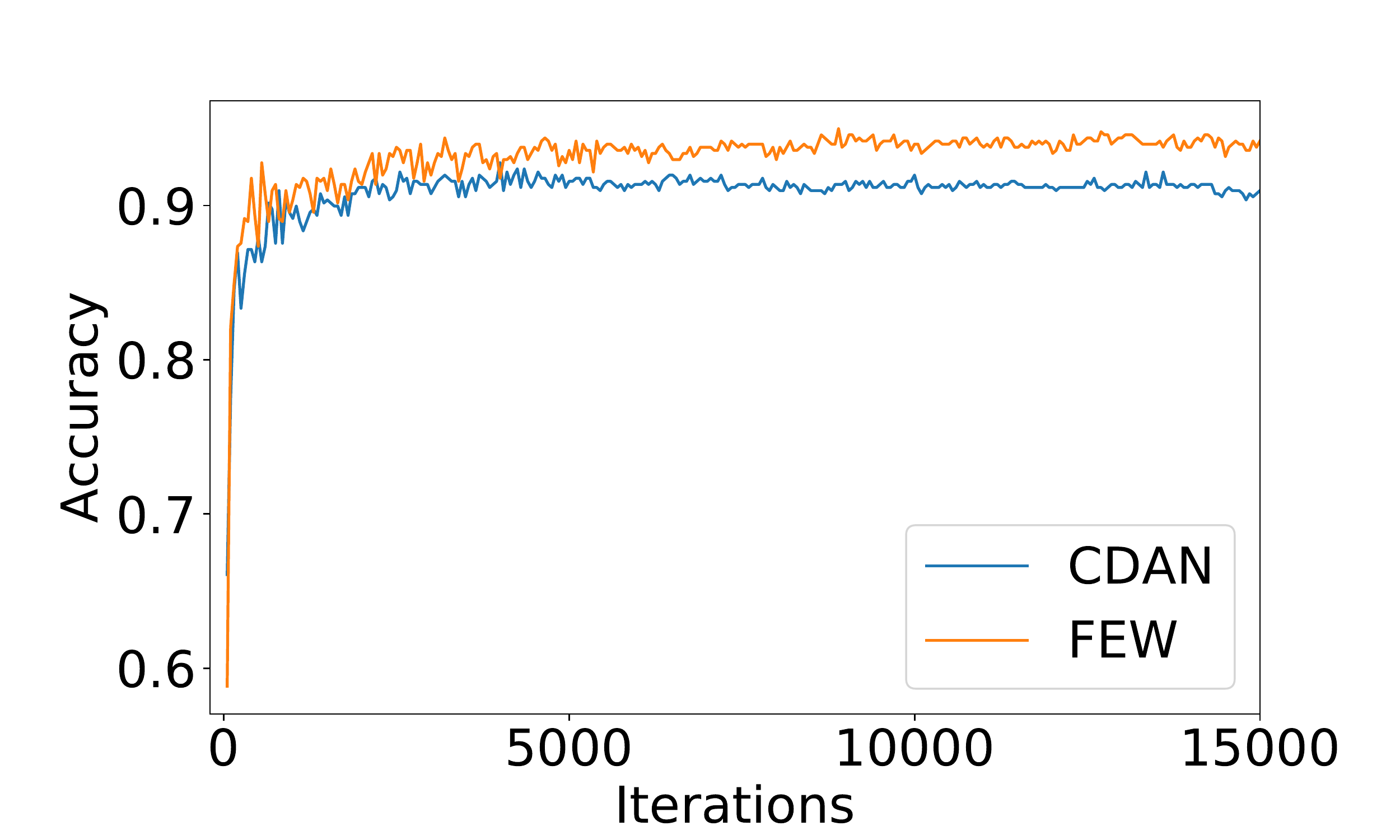}
		}
		\subfigure[\{Ar,Cl\}$\rightarrow$Pr]{
			\includegraphics[width=0.3\linewidth]{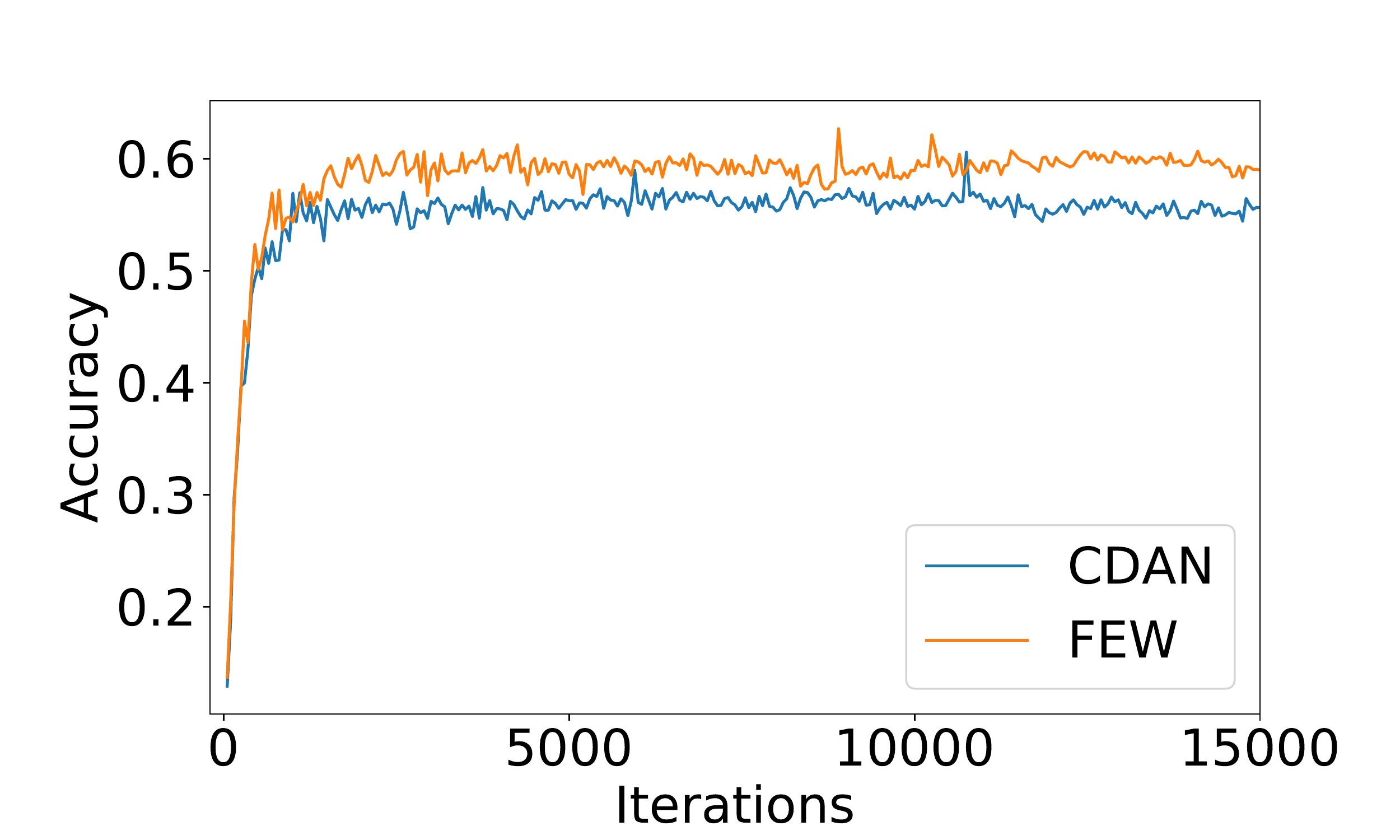}
		}
		\subfigure[\{Cl,Rw\}$\rightarrow$Ar]{
			\includegraphics[width=0.3\linewidth]{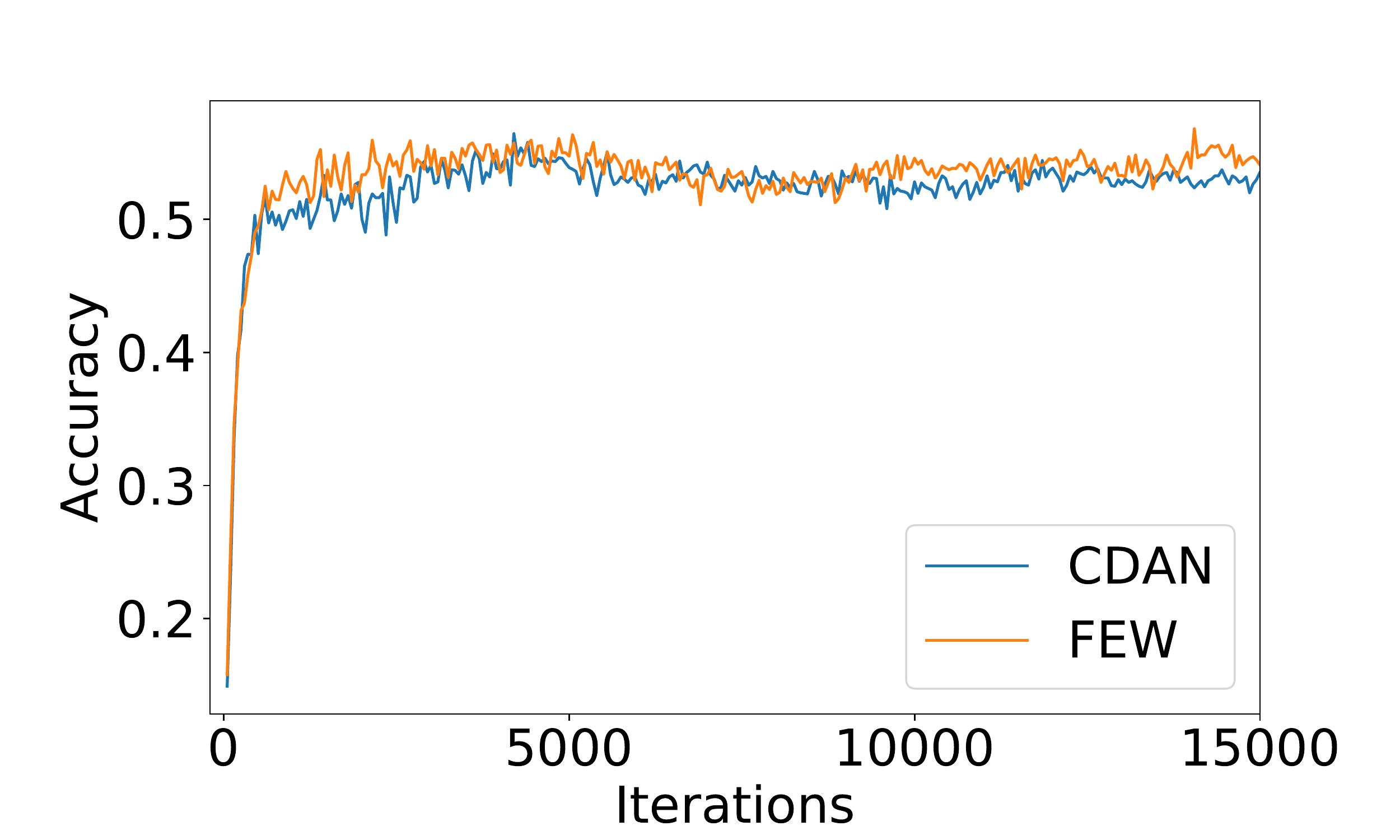}
		}
		\caption{(a), (b), and (c) are the accuracy curves of task \{A,W\}$\rightarrow$D, \{Ar,Cl\}$\rightarrow$Pr, and \{Cl,Rw\}$\rightarrow$Ar respectively. The orange curve denotes the accuracies achieved by the proposed method ``FEW'' of  each iteration. The blue curve denotes the accuracies of ``CDAN'' \cite{long2018conditional}.}\label{fig:convergence}
	\end{figure*}
	
	\begin{figure*}[t]
		\centering
		\includegraphics[width=0.9\linewidth]{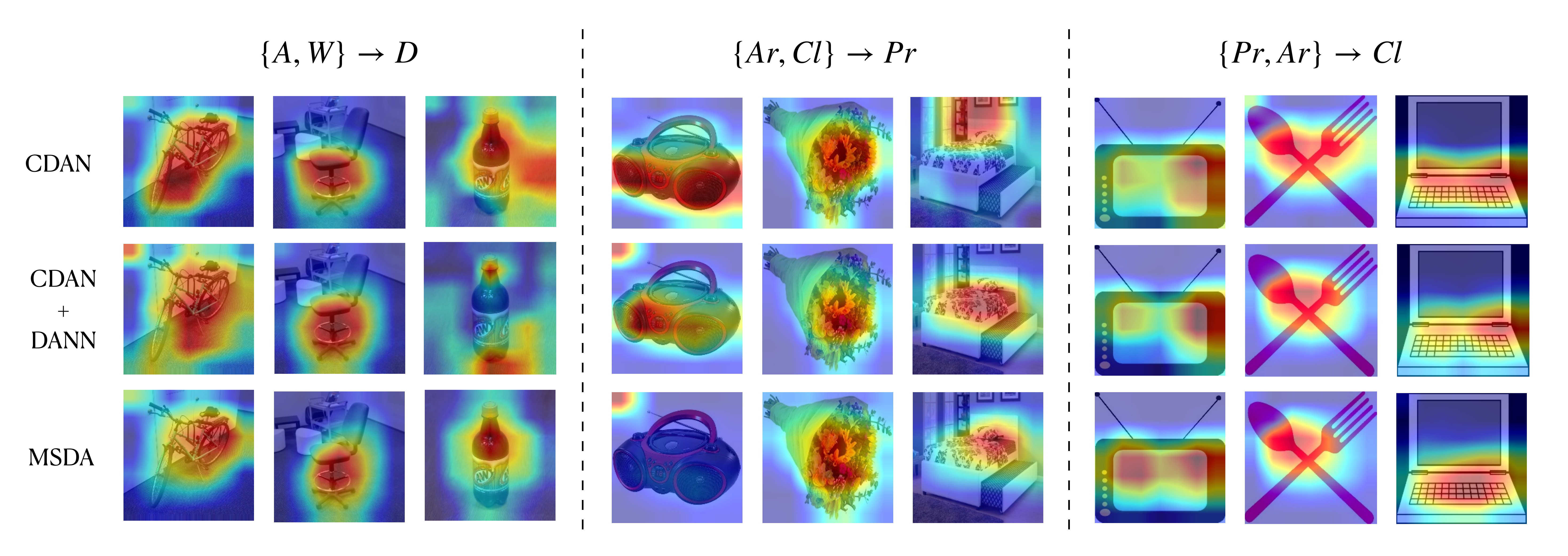}
		\caption{Heat maps of several selected examples under different methods: CDAN, CDAN+DANN, FEW. Examples are selected from 3 tasks: \{A,W\}$\rightarrow$D (left), \{Ar,Cl\}$\rightarrow$Pr (middle), and \{Pr,Ar\}$\rightarrow$Cl (right). Red regions correspond to high score for class. (Best viewed in color.)}\label{fig:grad_cam}
	\end{figure*}
	
	\subsection{Results}
	In this section, the experimental results under the given standard MSDA settings are reported first. Then, we also report results of some tasks evaluated under different category splitting strategy of source data. Following previous works \cite{long2018conditional}, each task was run for $3$ times and the average accuracies are reported in this paper. In this way, more than 300 times of experiments were implemented in our work.
	
	Four methods are selected as baselines: \textbf{1) ResNet50} \cite{he2016deep}: the ResNet50 model pretrained on ImagNet is fine-tuned on source data without additional adaptation algorithms; \textbf{2) DANN} \cite{ganin2016domain}: an adversarial learning method which reduces the marginal distribution discrepancy between source features and target features; \textbf{3) CDAN} \cite{long2018conditional}: using tensor products of features and probabilities as the inputs of the discriminator to bound the conditional distribution discrepancy between source data and target data in an adversarial manner; \textbf{4) CDAN+DANN}: this method can be seen as an \textbf{ablation} study of the proposed method. Here, we directly use DANN to reduce the distribution discrepancy between categories of domain $\alpha$ and categories of domain $\beta$ without using the feature element-wise weights ``$W$''. And CDAN is used to reduce the distribution discrepancy between source data and target data.
	
	\subsubsection{Under Standard Settings}
	The experimental results of Office-31 under standard settings of MSDA are listed in Table \ref{tab:office_resnet50} and those of Office-Home are reported in Table \ref{tab:officehome_resnet50}. The best results are bold. In order to compare MSDA with traditional domain adaptation, we also report experimental results for conventional domain adaptation, i.e., all categories of source data are from the same domain. A ``\ding{51}'' placed in column ``MSDA'' indicates that the method is applied under the settings of MSDA while a ``\ding{55}'' means that the results are obtained under the settings of the conventional domain adaptation.
	
	From Table \ref{tab:office_resnet50} and Table \ref{tab:officehome_resnet50} we can observe that the performance of previous methods decline on MSDA tasks compared with that on conventional domain adaptation tasks. For Office-31, the decline is not obvious since domains in Office-31 do not have large distribution discrepancy, especially for ``W'' and ``D''. For example, on task \{A,D\}$\rightarrow$W, traditional domain adaptation methods (i.e., ResNet50, DANN, CDAN) achieve higher accuracies than task A$\rightarrow$W under the effect of domain ``D''. Thus it is not easy to observe whether the performance is influenced by domain shift in source dataset from such kind of tasks. However, we can still observe that accuracies of traditional methods (i.e., ResNet, DANN, CDAN) on MSDA task \{W,D\}$\rightarrow$A are lower than those on both W$\rightarrow$A and D$\rightarrow$A, which reveals the influence of the domain shift inside source data. 
	
	For Office-Home, the decline in performance of previous methods can be more clearly observed. For example, for task \{Ar,Cl\}$\rightarrow$Pr, DANN's and CDAN's accuracies are lower than those of both Ar$\rightarrow$Pr and Cl$\rightarrow$Pr. The same phenomenon can be also observed on task \{Ar,Pr\}$\rightarrow$Cl in Table \ref{tab:officehome_resnet50}. Because of space limitation, we only report results of 6 traditional tasks, however, this performance declination can be observed on more tasks (as shown in Appendix B). Results on Office-Home further reveal that domain shift between different categories of source date can influence domain adaptation performance. Thus, MSDA problem is valid and needs to be explored.
	
	From Table \ref{tab:office_resnet50} and Table \ref{tab:officehome_resnet50}, we can see that the proposed ``FEW'' can achieve better performance on $15$ tasks out of $18$ than previous typical domain adaptation methods, which indicates the effectiveness of ``FEW''. Then, to prove the significance of feature element-wise weight ``W'', we directly use DANN without weight $W$ when training $D_1$ and report the results in Table \ref{tab:office_resnet50} and Table \ref{tab:officehome_resnet50}, denoted as ``CDAN+DANN''. It can be observed that the performance of ``FEW'' is better than or equal to ``CDAN+DANN'' on $17$ tasks out of $18$. Thus, we can infer that the feature element-wise weight ``W'' can help to alleviate the domain shift between different categories and cause less category confusion than ``DANN''. We also apply ``FEW'' to traditional domain adaptation tasks, where we use it to reduce the distribution discrepancy between class $1\sim 20$ (resp. $1\sim40$) and class $21\sim31$ (resp. $41\sim65$) of the source domain for Office-31 (resp. Office-Home). When being applied on traditional tasks, ``FEW'' dose not show improvements on most tasks, which again indicates that the proposed method achieves better performance because of reduced domain discrepancy instead of category confusion since it does not work well when the categories are from the same domains. However, the proposed method has its limitation. For some tasks like \{Rw,Cl\}$\rightarrow$Ar in Table \ref{tab:officehome_resnet50}, performance of ``FEW'' is worse than traditional method ``CDAN''. This may be because that though we try to use $W$ to map feature elements that contain more domain information close, still features of the last layer include too much content information. Therefore, ``FEW'' can not work well on some tasks because of category confusion. Features that contain more domain information need to be explored in the future for MSDA.
	
	\subsubsection{Different Splitting Settings}
	Experimental results on several domain adaptation tasks under different category splitting settings of source data are reported in Table \ref{tab:office_diff}. For Office-31, category $1\sim 6$ are from domain $\alpha$ and category $7\sim 31$ are from domain $\beta$. For Office-Home, category $1\sim 10$ are from domain $\alpha$ while category $11\sim 65$ come from domain $\beta$. $4$ tasks are selected to report here: \{A,D\}$\rightarrow$W, \{W,D\}$\rightarrow$A, \{Rw,Cl\}$\rightarrow$Ar, and \{Pr,Rw\}$\rightarrow$Ar. More results could be found in Appendix C. Table \ref{tab:office_diff} shows that ``FEW'' can outperform previous method CDAN and method without ``$W$'', i.e., CDAN+DANN.

	\subsection{Analysis}
	\subsubsection{Convergence}
	In order to show the convergence of the proposed method, accuracies of each iteration achieved by models trained using traditional ``CDAN'' and the proposed ``FEW'' are plotted in Figure \ref{fig:convergence} for comparison. Here $3$ tasks are selected as examples: \{A,W\}$\rightarrow$D, \{Ar,Cl\}$\rightarrow$Pr, and \{Cl,Rw\}$\rightarrow$Ar as shown in Figure \ref{fig:convergence} (a), (b), and (c) respectively. From Figure \ref{fig:convergence}, we can observe that the proposed method converges well, which reveals that reducing distribution discrepancy between different categories of source data using the proposed method will not influence the convergence compared with using CDAN only.
	
	\subsubsection{Heat Map}
	In order to show the effect of the proposed method more intuitively, we use heat maps to show which area of the images influence the predicted results most, i.e., which pixels the models focus on. Here we use the deep neural network visualization method ``Grad-Cam'' \cite{selvaraju2017grad} to add a mask to original images. Different colors are used for different pixels according to their importance, decreasing from red to blue. Some examples are shown in Figure \ref{fig:grad_cam}. Here, we plot the heat maps of $3$ methods: CDAN, CDAN+DANN, and FEW. And for each method, $3$ tasks are evaluated: \{A,W\}$\rightarrow$D (left panel), \{Ar,Cl\}$\rightarrow$Pr (middle panel), and \{Pr,Ar\}$\rightarrow$Cl (right panel). Categories from left to right are: bike, chair, bottle, radio, flower, bed, TV, spoon, and laptop respectively. We can observe that in most cases, FEW can make the model focus more on the objects than other areas, which can reflect the effectiveness of the FEW method to some degree. For example, in task \{A,W\}$\rightarrow$D, model trained by FEW focuses more on the bike instead of the floor. Similar phenomena can also be observed for category bottle, bed, spoon, laptop, etc. However, FEW may cause misclassification sometimes, e.g., the $1$st column of task \{Ar,Cl\}$\rightarrow$Pr in Figure \ref{fig:grad_cam}. And for ``CDAN+DANN'' which maps different categories close without using $W$, it does not show its effectiveness in making the model focus on the object in most cases like bike, chair, bottle, TV, spoon, and laptop. This reveals the importance of exploiting the element-wise weight $W$.
	
	\section{Conclusion}
	In this paper, we propose a new type of domain adaptation: Mixed Set Domain Adaptation (MSDA). Under the settings of MSDA, category $1\sim k$ and category $k+1 \sim c$ of the source data are from different domains, which is more common in reality compared with conventional domain adaptation settings. Qualitative analysis and quantitative results show that the performance of typical domain adaptation methods can be influenced by domain shift between different categories of source dataset, which indicates the significance of proposing MSDA. We also propose a feature element-wise weighting (FEW) method to address MSDA problems. Experimental results show the effectiveness of the proposed method. Features that contain more domain information than content information still need to be explored for MSDA in the future.
	
	\cleardoublepage
	\bibliography{Msda}
	\clearpage
	\begin{appendices}
		\section{Appendix Overview} \label{app:op_proof}
		In the appendix, we provide the proof of the optimum achieved by FEW in Appendix A. Then in Appendix B, we report the complete experimental results including all traditional tasks for Office-Home. In Appendix C, more experimental results are given under different category splitting strategy of source data. Then, heat maps of more examples are plotted in Appendix D.
		
		\section{Appendix A. Optimum for element-wise weighted adversarial learning}
		In this section, we give the proof for the optimum of the feature element-wise weighted adversarial learning.
		
		\begin{theorem}
			Given the objective function $V(G,D_1)$ in eq. \ref{eq:mar_adv_weight} and following the proof of ``Proposition 1'' in~\cite{goodfellow2014generative}, \textit{for any fixed $G$, the optimal discriminator $D_1$ in eq. \ref{eq:mar_adv_weight} is }
			\begin{equation}
			D^{*}_{G}(W\odot f) = \frac{G_{x^\alpha}(W\odot f)}{G_{x^\alpha}(W\odot f) + G_{x^\beta}(W\odot f)}
			\label{eq:op_D}
			\end{equation}
			where $G_{x^\alpha}(W\odot f)$ denotes the distribution of the element-wise weighted deep features of $X^\alpha$ while $G_{x^\beta}(W\odot f)$ is that of $X^\beta$. The deep features are denoted by $f$.
			
		\end{theorem}
		
		\begin{proof}	
			For eq. \ref{eq:mar_adv_weight}, given any fixed generator $G$, the discriminator $D_1$ is trained to maximize the value function $V(G,D_1)$:
			\begin{equation}
			\begin{aligned}
			\min\limits_G \max\limits_{D_1} & V(G,D_1) \\
			= & \int_{x^\alpha} p_\alpha (x^\alpha)[logD_1 (W\odot G^{f}(x^\alpha))]d_{x^\alpha} \\
			+ & \int_{x^\beta} p_\beta (x^\beta)[log(1-D_1 (W\odot G^{f}(x^\beta)))]d_{x^\beta} \\
			= & \int_{W\odot f} G_{x^\alpha} (W\odot f)[logD_1 (W\odot f)] \\
			+ & G_{x^\beta} (W\odot f)[log(1-D_1 (W\odot f))]d_{W\odot f} \\
			\end{aligned}
			\label{eq:optimal_D}
			\end{equation}
        where $G_{x^\alpha}(W\odot f) = (W\odot f)_{x^\alpha \sim p_\alpha (x)}$ and $G_{x^\beta}(W\odot f) = (W\odot f)_{x^\beta \sim p_\beta (x)}$. Eq.~\ref{eq:optimal_D} has the same form as function $y \rightarrow a\, log(y) + b\,log(1 - y), (a,b) \in \mathbb{R}^2 \setminus \{0,0\}$, which achieves its maximum at $\frac{a}{a + b} \in [0,1]$. So similarly, given $G$ fixed, the optimal $D$ that makes $V(G,D_1)$ achieve its maximum can be obtained as in eq.~\ref{eq:op_D}. 
		\end{proof}
		Then, by substituting eq. \ref{eq:op_D} into eq. \ref{eq:mar_adv_weight}, the training criterion for $G$ is to minimize
		
		\begin{equation}
		\begin{aligned}
		& V(G,D^{*}_{G}) \\
		& = \mathbb{E}_{x^{\alpha}\sim p_{\alpha}(x)}[logD^{*}_{G}(W\odot G^{f}(x^{\alpha}))] \\
		& + \mathbb{E}_{x^{\beta}\sim p_{\beta}(x)}[log(1-D^{*}_{G}(W\odot G^{f}(x^{t})))] \\
		& = \mathbb{E}_{W\odot f \sim G_{x^\alpha}(W\odot f)}[logD^{*}_{G}(W\odot f)] \\
		& + \mathbb{E}_{W\odot f \sim G_{x^\beta}(W\odot f)}[log(1-D^{*}_{G}(W\odot f))] \\
		& = \mathbb{E}_{W\odot f \sim G_{x^\alpha}(W\odot f)}[log\frac{G_{x^\alpha}(W\odot f)}{G_{x^\alpha}(W\odot f) + G_{x^\beta}(W\odot f)}] \\
		& + \mathbb{E}_{W\odot f \sim G_{x^\beta}(W\odot f)}[log\frac{G_{x^\beta}(W\odot f)}{G_{x^\alpha}(W\odot f) + G_{x^\beta}(W\odot f)}] \\
		\end{aligned}
		\label{eq:min_G}
		\end{equation}
		According to~\cite{goodfellow2014generative}, it is straightforward to induce that eq.~\ref{eq:min_G} can be reformulated to
		\begin{equation}
		V(G,D^{*}_{G}) = -log(4) + 2\cdot JSD(G_{x^\alpha}(W\odot f)\parallel G_{x^\beta}(W\odot f))
		\end{equation}
		We can see that when $G_{x^\alpha}(W\odot f) = G_{x^\beta}(W\odot f)$, the global minimum can be achieved as the Jensen-Shannon divergence (JSD) between two distributions is always non-negative and equals to zero iff they are exactly the same. To sum up, in the element-wise weighted adversarial architecture, the deep neural network $G$ tends to generate:
		\begin{equation}
		\begin{aligned}
		G_{x^\alpha}(W\odot f) = G_{x^\beta}(W\odot f).
		\end{aligned}
		\label{eq:distribution}
		\end{equation} 
		
		\section{Appendix B. Complete Experimental Results on Office-Home}
		Here we display the complete experimental results for Office-Home in Table \ref{tab:officehome_add}, including all traditional tasks. We could observe performance decline of traditional domain adaptation methods from Table \ref{tab:officehome_add}. For example, CDAN and DANN perform worse on task \{Cl,Rw\}$\rightarrow$Ar than that on task Cl$\rightarrow$Ar and Rw$\rightarrow$Ar. Also, CDAN and DANN perform worse on task \{Cl,Pr\}$\rightarrow$Ar than that on task Cl$\rightarrow$Ar and Pr$\rightarrow$Ar. For task \{Pr,Cl\}$\rightarrow$Rw, CDAN perform worse than task Pr$\rightarrow$Rw and Cl$\rightarrow$Rw. More similar cases could be observed in Table \ref{tab:officehome_add}.
			\begin{table*}[t]
			\setlength{\belowcaptionskip}{0.5cm}
			\centering
			\caption[LoF entry]{Accuracy on Office-Home under standard settings of Mixed set domain adaptation.}
			\renewcommand\arraystretch{1.0}
			\small
			\begin{tabular}{p{2.3cm}<{\centering}|p{0.8cm}<{\centering}|p{1.5cm}<{\centering}p{1.5cm}<{\centering}p{1.5cm}<{\centering}p{1.5cm}<{\centering}p{1.5cm}<{\centering}p{1.6cm}<{\centering}|p{0.8cm}<{\centering}}
				\hline
				Method & MSDA & Ar$\rightarrow$Cl & Pr$\rightarrow$Cl & Ar$\rightarrow$Pr & Cl$\rightarrow$Ar & Cl$\rightarrow$Pr &  Cl$\rightarrow$Rw & Avg. \\
				\hline
				ResNet50 & \ding{55} & $45.1$ & $41.4$ & $64.9$ & $54.3$ & $60.4$ & $62.8$ & $54.8$ \\
				DANN & \ding{55} & $47.3$ & $45.8$ & $64.6$ & $53.7$ & $62.1$ & $64.1$ & $56.3$  \\			
				CDAN & \ding{55} & $49.0$ & $48.3$ & $69.3$ & $54.4$ & $66.0$ & $68.4$ & $59.2$ \\
				FEW & \ding{55} & $49.1$ & $44.0$ & $64.6$ & $51.2$ & $61.8$ & $60.4$ & $55.2$ \\
				\hline		
				Method & MSDA & \{Ar,Cl\}$\rightarrow$Pr & \{Ar,Pr\}$\rightarrow$Cl & \{Cl,Pr\}$\rightarrow$Ar & \{Cl,Rw\}$\rightarrow$Ar & \{Rw,Cl\}$\rightarrow$Ar & \{Pr,Cl\}$\rightarrow$Rw & Avg. \\
				\hline
				ResNet50 & \ding{51} & $61.0\pm0.7$ & $42.5\pm1.0$ & $51.1\pm0.6$ & $54.6\pm0.2$ & $59.2\pm0.4$ & $63.9\pm0.3$ & $51.7$ \\
				DANN & \ding{51} & $60.3\pm0.4$ & $45.6\pm0.2$ & $47.3\pm0.4$ & $55.8\pm0.7$ & $57.2\pm0.5$ & $65.8\pm0.5$ & $55.3$  \\			
				CDAN & \ding{51} & $59.7\pm0.9$ & $45.9\pm0.5$ & $52.0\pm0.5$ & $56.4\pm0.2$ & $\bm{63.9\pm0.3}$ & $67.8\pm0.6$ & $57.6$ \\
				\hline
				CDAN+DANN & \ding{51} & $62.4\pm0.7$ & $\bm{47.1\pm0.2}$ & $\bm{52.4\pm0.4}$ & $56.3\pm0.6$ & $62.2\pm0.3$ & $67.7\pm0.3$ & $58.0$ \\
				FEW & \ding{51} & $\bm{62.7\pm0.7}$ & $\bm{47.1\pm0.8}$ & $\bm{52.4\pm0.6}$ & $\bm{56.8\pm0.1}$ & $62.8\pm0.3$ & $\bm{68.3\pm0.2}$ & $\bm{58.4}$ \\
				\hline
				\multicolumn{9}{c}{} \\
				\multicolumn{9}{c}{} \\
				\hline
				Method & MSDA & Ar$\rightarrow$Rw & Pr$\rightarrow$Ar & Pr$\rightarrow$Rw & Rw$\rightarrow$Ar & Rw$\rightarrow$Cl &  Rw$\rightarrow$Pr & Avg. \\
				\hline
				ResNet50 & \ding{55} & $71.8$ & $49.7$ & $71.5$ & $66.7$ & $51.0$ & $78.2$ & $64.8$ \\
				DANN & \ding{55} & $70.2$ & $52.7$ & $68.7$ & $63.9$ & $54.0$ & $79.4$ & $64.8$  \\			
				CDAN & \ding{55} & $74.5$ & $55.6$ & $75.9$ & $68.4$ & $55.4$ & $80.5$ & $68.3$ \\
				FEW & \ding{55} & $70.8$ & $50.2$ & $70.8$ & $67.3$ & $49.9$ & $77.5$ & $64.4$ \\
				\hline		
				\hline
				Method & MSDA & \{Pr,Ar\}$\rightarrow$Cl & \{Rw,Pr\}$\rightarrow$Ar & \{Ar,Rw\}$\rightarrow$Cl & \{Cl,Rw\}$\rightarrow$Pr & \{Pr,Rw\}$\rightarrow$Ar & \{Rw,Ar\}$\rightarrow$Cl & Avg. \\
				\hline			
				CDAN & \ding{51} & $46.0\pm0.7$ & $\bm{62.3\pm0.3}$ & $47.6\pm0.5$ & $66.4\pm1.3$ & $58.2\pm0.7$ & $51.4\pm0.5$ & $55.3$ \\
				\hline
				CDAN+DANN & \ding{51} & $\bm{47.1\pm0.8}$ & $61.8\pm0.6$ & $\bm{48.2\pm0.3}$ & $65.8\pm0.3$ & $57.5\pm0.6$ & $48.7\pm0.8$ & $54.9$ \\
				FEW & \ding{51} & $\bm{47.1\pm0.6}$ & $62.0\pm0.3$ & $\bm{48.2\pm0.3}$ & $\bm{67.1\pm0.6}$ & $\bm{58.5\pm0.2}$ & $\bm{51.5\pm0.3}$ & $\bm{55.7}$ \\
				\hline
			\end{tabular}
			\label{tab:officehome_add}
		\end{table*}
		
		\section{Appendix C. More Experiments under Different Settings}
		In this section, we report experimental results for more tasks under different category splitting strategies in Table \ref{tab:officehome_diff}. For Office-31, class $1\sim5$ are from domain $\alpha$ while class $6\sim31$ are from domain $\beta$. For Office-Home, class $1\sim10$ are from domain $\alpha$ while class $11\sim65$ are from domain $\beta$. Task \{D,A\}$\rightarrow$W, \{W,A\}$\rightarrow$D, \{Cl,Rw\}$\rightarrow$Ar, and \{Cl,Rw\}$\rightarrow$Pr are reported in this section. We can observe that the proposed ``FEW'' can outperform CDAN and CDAN+DANN which reveals the effectiveness of its effectiveness.
		\begin{table*}[t]
			\centering
			\caption[LoF entry]{Accuracy of more tasks under different category splitting strategies of MSDA.}
			\small
			\begin{tabular}{p{2.0cm}<{\centering}p{1.5cm}<{\centering}p{1.5cm}<{\centering}p{1.0cm}<{\centering}|p{1.6cm}<{\centering}p{1.6cm}<{\centering}p{1.6cm}<{\centering}}
				\hline
				\multirow{2}{*}{Method} &
				\multicolumn{3}{c|}{ Office-31} & \multicolumn{3}{c}{ Office-Home }  \cr 
				& \{D,A\}$\rightarrow$W & \{W,A\}$\rightarrow$D & Avg. & \{Cl,Rw\}$\rightarrow$Pr &  \{Cl,Rw\}$\rightarrow$Ar & Avg. \\
				\hline
				CDAN & $90.3\pm 0.3$ & $88.0\pm 0.9$ & $89.2$ & $75.0\pm 0.3$ & $\bm{62.5\pm0.3}$ & $68.8$  \\
				CDAN+DANN & $\bm{90.9\pm 0.6}$ & $89.3\pm0.4$ & $90.1$ & $75.3\pm 0.3$ & $62.1\pm 0.1$ & $68.7$ \\
				FEW & $\bm{90.9\pm 0.3}$ & $\bm{89.6\pm0.5}$ & $\bm{90.3}$ & $\bm{75.5\pm 0.2}$ & $62.2\pm 0.1$ & $\bm{68.9}$ \\
				\hline
			\end{tabular}
			\label{tab:officehome_diff}
		\end{table*}
	
		\begin{figure}[t]
		\centering
		\subfigure[\{W,A\}$\rightarrow$D]{
			\includegraphics[width=0.6\linewidth]{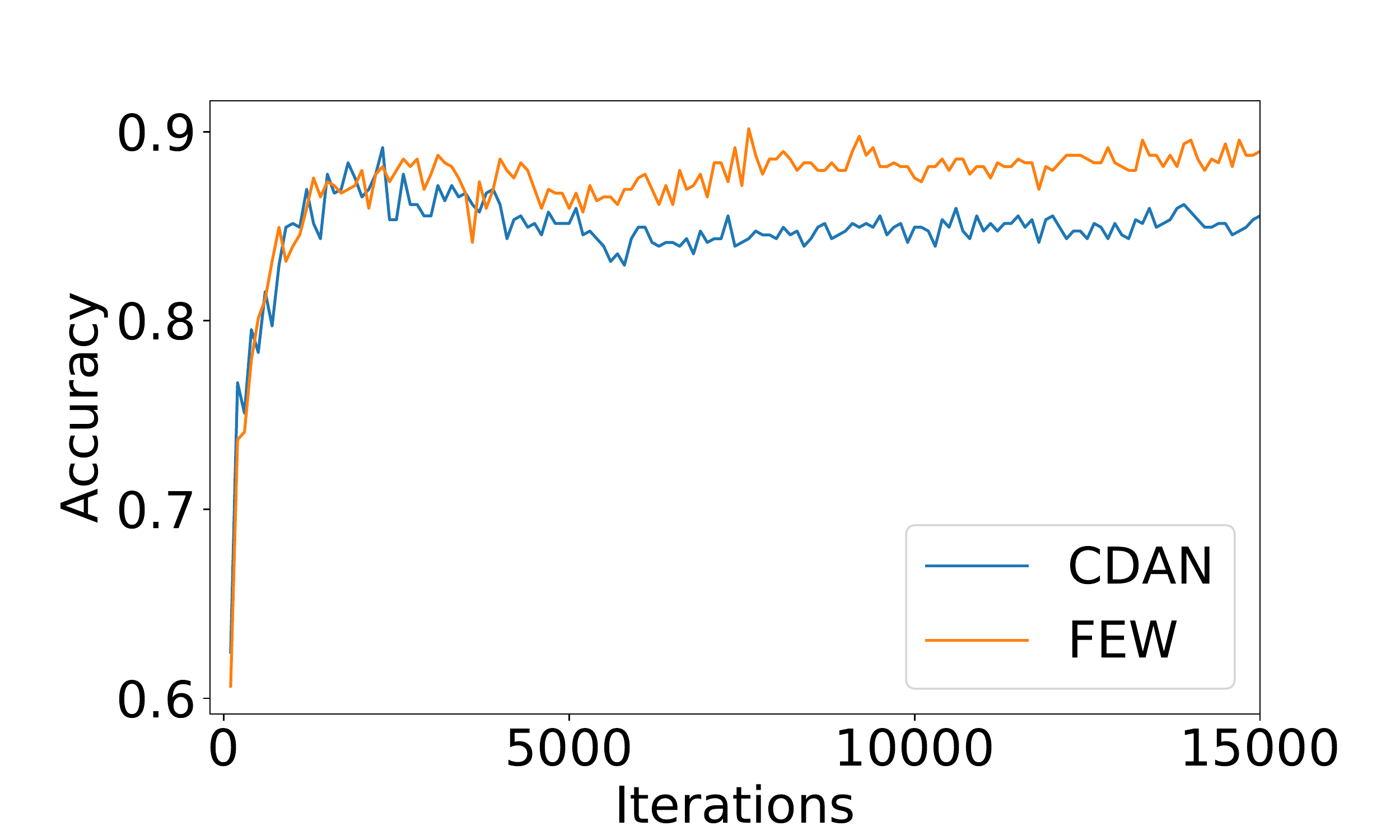}
		} \\
		\subfigure[\{Cl,Rw\}$\rightarrow$Ar]{
			\includegraphics[width=0.6\linewidth]{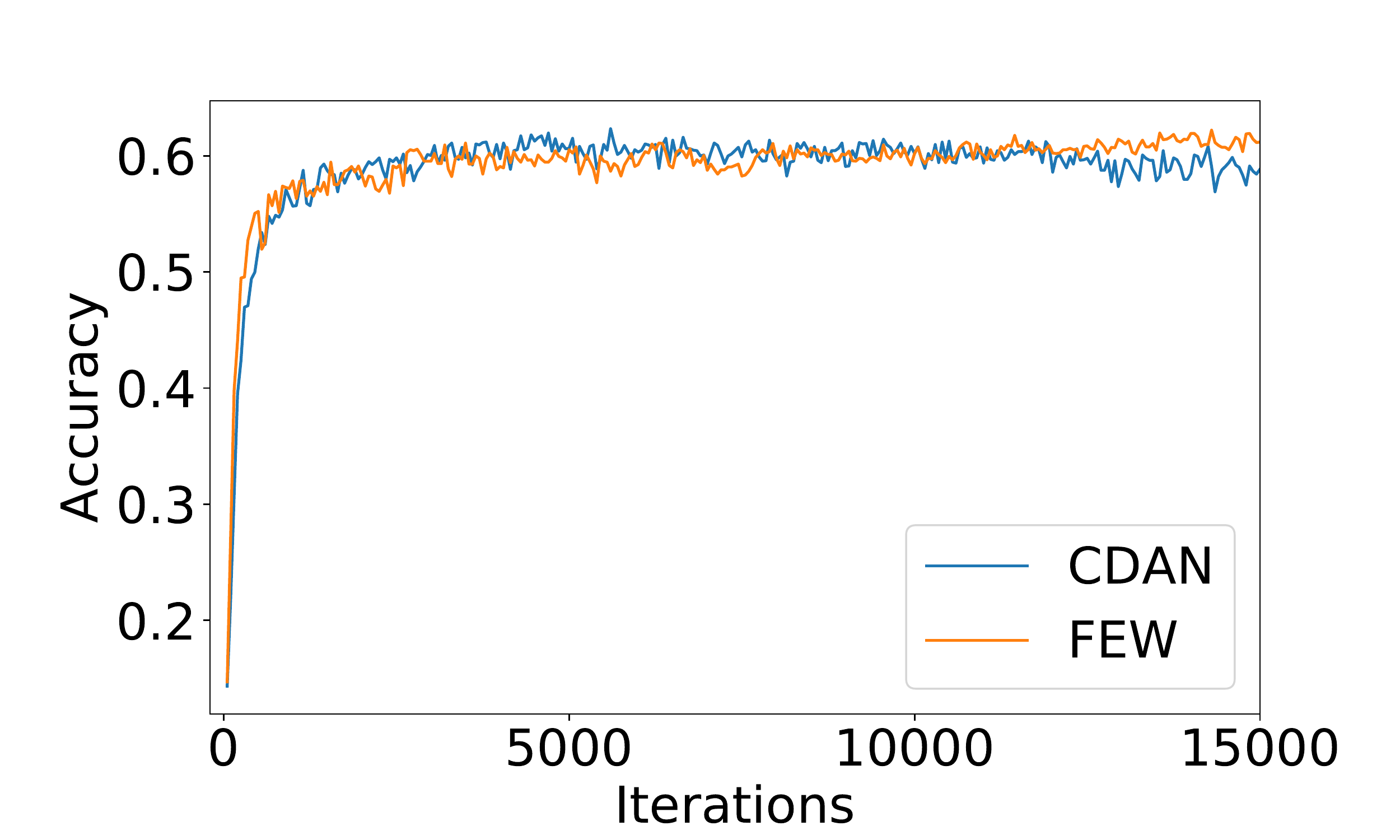}
		}
		\caption{Accuracy curves of task \{W,A\}$\rightarrow$D and \{Cl,Rw\}$\rightarrow$Ar under different splitting strategies of data set Office-31 and Office-Home respectively. }\label{fig:acc_diff}
    	\end{figure}
		In addition, here we plot the accuracy curves for task \{W,A\}$\rightarrow$D and \{Cl,Rw\}$\rightarrow$Ar under different splitting strategies of the source domain in Figure \ref{fig:acc_diff}. From Figure \ref{fig:acc_diff}, we can observe that the proposed ``FEW'' method are more stable than traditional domain adaptation method ``CDAN' when dealing with MSDA problems. For task \{W,A\}$\rightarrow$D, FEW can achieve higher accuracy and its accuracy keeps rising steadily while that of CDAN declines after a number of iterations. As for task \{Cl,Rw\}$\rightarrow$Ar, the best accuracy of CDAN are higer than that of FEW. However, we can observe that FEW converges to a higher accuracy than CDAN, which indicates that the proposed FEW is more effective.
		
		\section{Appendix D. Heat Map}
		Here we plot ``heat maps'' of more examples in Figure \ref{fig:heat_maps_appendix} in order to show the effect of our proposed method. ``Grad-cam'' method is applied to models trained on CDAN (left), CDAN+DANN (middle) and the proposed method FEW (right) respectively for comparison. Examples are selected from 5 tasks: \{W,D\}$\rightarrow$A, \{A,D\}$\rightarrow$W, \{Ar,Pr\}$\rightarrow$Cl, \{Cl,Pr\}$\rightarrow$Ar, and \{Rw,Cl\}$\rightarrow$Ar, from the first row to the last respectively. For each task, 3 images are selected to plot. From Figure \ref{fig:heat_maps_appendix}, we can observe that models trained by using FEW can focus more on the object than CDAN and CDAN+DANN for most cases.
		\begin{figure*}
			\centering
			\includegraphics[width=\linewidth]{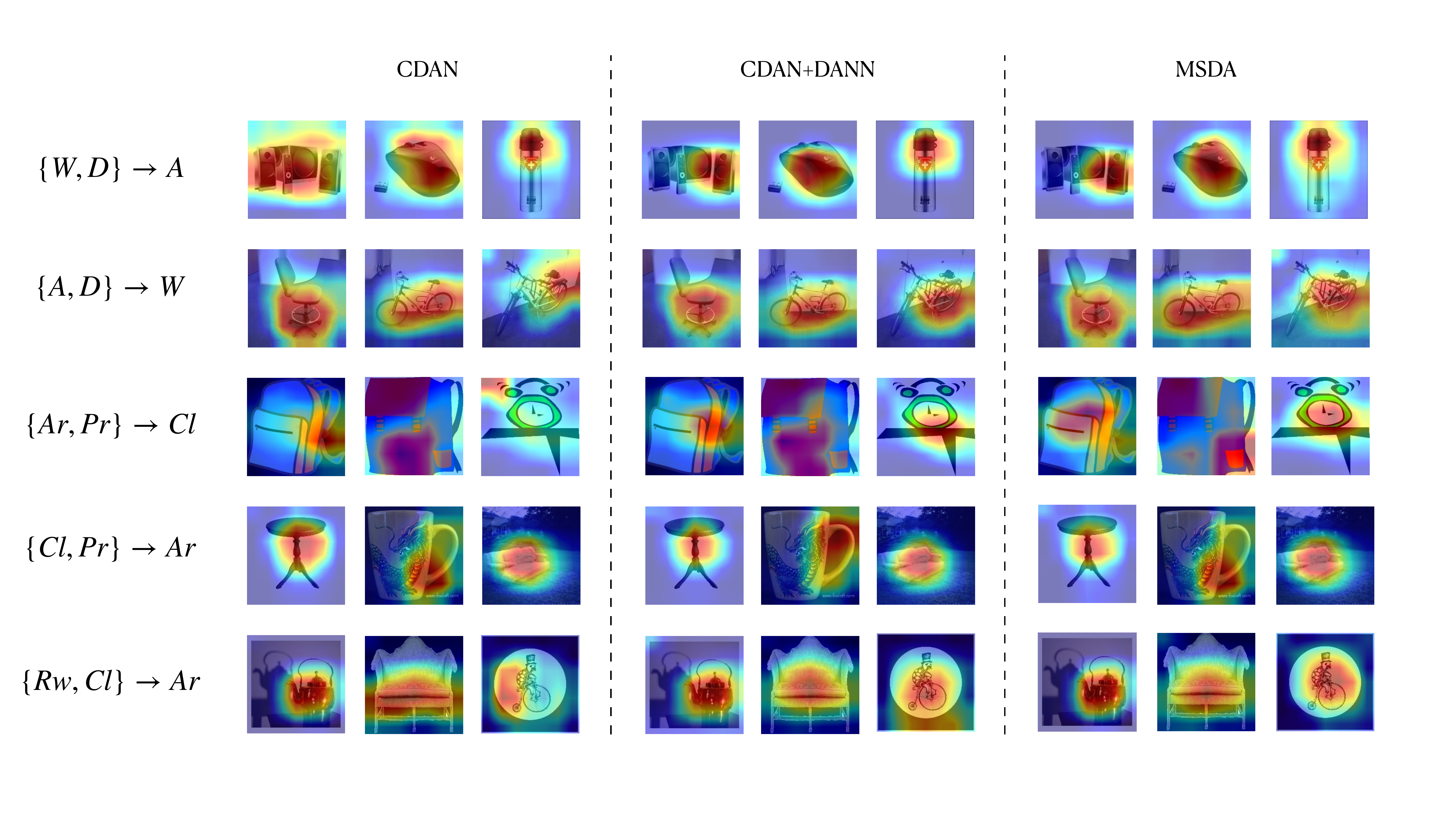}
			\caption{Heat maps. Plotted by using Grad-Cam \cite{selvaraju2017grad}.}\label{fig:heat_maps_appendix}
		\end{figure*}
	\end{appendices}
\end{document}